%% file: deltareplay_iclr2027_v7.tex
\documentclass{article} 
\usepackage{iclr2027_conference,times}

\input{math_commands.tex}

\usepackage{hyperref}
\usepackage{url}
\usepackage{booktabs}
\usepackage{amsmath}
\usepackage{amssymb}
\usepackage{amsthm}
\usepackage{graphicx}
\usepackage{xcolor}
\usepackage{multirow}
\usepackage{enumitem}
\usepackage{algorithm}
\usepackage{algpseudocode}
\usepackage{microtype}
\usepackage{listings}
\usepackage{textcomp}
\usepackage{placeins}

\newcommand{\ours}{DeltaReplay}
\newcommand{\cmark}{$\checkmark$}
\newcommand{\xmark}{$\times$}
\newcommand{\pmark}{$\triangle$}
\newtheorem{proposition}{Proposition}

\title{DeltaReplay: Task-Relative Memory Reuse \\ for Mobile GUI Agents}

\author{Yudong Bai, Yihong Chen, Quanming Yao\thanks{Corresponding author.}\hspace{0.35em}, Yaqing Wang \\ Tsinghua University}

\iclrfinalcopy 
\begin{document}

\maketitle
\lhead{Preprint. Under review.}

\begin{abstract}
Memory-augmented mobile GUI agents store successful execution trajectories and reuse them in later tasks, but a stored trajectory rarely matches a new task exactly.
The new task may use different parameters, share only some of its steps with a stored trajectory, or have no relevant record in memory.
Forcing the agent to use irrelevant memory can mislead it, whereas discarding memory that may still be useful deprives it of guidance from past experience.
To address this dilemma, we propose DeltaReplay, a step-level memory reuse framework that decides how to use existing memory without modifying it.
We observe that the reusable part of a stored record is determined not by the record itself but by its relation to the new task, mainly through two factors: page-level consistency and action-level generality.
We therefore store execution trajectories as paths in a transition graph, whose nodes (pages) and edges (actions between pages) capture these two factors. At reuse time, the action on each edge is split into a task-independent operation and task-specific parameters.
DeltaReplay then compares each recorded step with the new task and the current screen, and decides whether to follow it, execute it after replacing its parameters, or leave it to the base agent.
On AndroidWorld and SPA-Bench, DeltaReplay improves the task success rate over a base agent with the same backbone by up to 10.3 and 25.0 percentage points, respectively.
These results indicate that deciding at each step how to use retrieved memory lets agents benefit even from partially matching trajectories.
\end{abstract}

\section{Introduction}
\label{sec:intro}

Mobile GUI agents built on multimodal large language models (MLLMs) read screens, follow natural-language instructions,
and act through taps, text entry, and swipes~\citep{zhang2025appagent,wang2024mobile,hong2024cogagent,qin2025uitars}.
Many of the tasks they receive recur: users file the same kind of expense, message the same contacts, and toggle the same settings.
An agent that stores its successful runs can reuse them instead of reasoning from scratch,
and a growing line of work equips GUI agents with trajectory, workflow, and graph memories for this purpose~\citep{wen2024autodroid,lee2024mobilegpt,wang2025agentworkflow,chen2025pgagent,xie2025guiexplorer,zhong2026actionengine}.

However, a stored trajectory rarely fits a new task exactly.
The new task may use different parameters, such as another contact or amount, or share only some of its steps with the stored task,
and frequent app updates can change the screens on which recorded actions were taken~\citep{42matters2026update,mcilroy2016fresh,lu2025transbench}.
A retrieved step may therefore be executable as recorded, useful only after being re-grounded on the current screen, or not applicable at all,
and each case calls for a different use: replay, re-ground, or discard.
Retrieving the right memory is thus only the beginning of using it: the agent must still decide, at each step, which of these cases applies.

Most memory-augmented GUI agents support only some of these uses, and fix them in advance (Table~\ref{tab:operators}):
some always add retrieved memory to the prompt~\citep{xie2025guiexplorer,chen2025pgagent,guan2025kgrag,sun2026magnet},
some choose only between executing a stored step and discarding it~\citep{jiang2025appagentx,mi2026darwinian,liu2025mobimem},
and some choose only between re-grounding and discarding~\citep{chai2026uikobe}. Each missing use causes a kind of loss.
\textbf{(i) Slow reuse: replayable steps are still reasoned through.}
Reasoning through a step is slow; GUI-explorer, for example, reports 66\,s per step with GPT-4o~\citep{xie2025guiexplorer}.
Yet agents that add memory to the prompt, or re-ground every step, do so even for steps that could be replayed as recorded.
Replaying them saves this time: on AndroidWorld tasks whose exact goal is already in memory, our method replays about half of the steps and takes 9.3\,s per step, against 19.0--32.2\,s for prompting-based memory baselines (Section~\ref{sec:coverage}).
\textbf{(ii) Over-discarding: partly usable memory is dropped as a whole.}
Execute-or-discard agents check whether a record can be used before reusing it, for example by asking an LLM whether a shortcut's preconditions hold~\citep{wang2025mobileagente}
or whether the record applies~\citep{jiang2025appagentx}, or by matching elements in the UI hierarchy~\citep{liu2025mobimem};
a record that passes is executed as a whole, and one that fails is discarded as a whole.
Yet a record often differs from the new task in only a few steps, for example in one typed value or one target label.
For such a record, deciding on the whole record is wrong either way: accepting it executes the old parameters as recorded, and rejecting it also throws away the steps that could still be followed.
In our experiments, records of the same template that differ from the task only in parameters solve 9.1\% of the tasks when replayed verbatim, and ignoring them leaves the base agent at 63.6\%,
whereas following their matching steps and replacing only their parameters reaches 86.4\% (Section~\ref{sec:ablation}).
\textbf{(iii) Misuse: inapplicable memory is still applied.}
LLM agents may reproduce retrieved experience even when it does not correspond to the current task~\citep{xiong2025memory};
in an ablation of Darwinian Memory, searching memory with subtasks keyed by the task goal alone lowers AndroidWorld success to 29.7\%,
11.3 percentage points below the 41.0\% obtained without memory~\citep{mi2026darwinian}.
In our experiments, when memory holds only other, weakly related tasks of the same app, adding retrieved records to the prompt lowers success from 39.6\% to 25.0\% (Section~\ref{sec:coverage}).

These failures mostly share one cause: existing methods treat a memory as a whole, using or discarding it together, whereas reusability varies across the steps of the same memory.
A record's reusable steps depend on the new task and the current screen, so they can only be decided step by step at use time.
We therefore present \ours{}, which makes the most of existing memory without modifying it and chooses among the three uses at every step (Figure~\ref{fig:method}).
\emph{Route~A (replay)} executes the recorded step once a structural check confirms that the screen matches, the recorded target maps to a current element, and the step carries no parameters of the old task.
\emph{Route~B (re-ground)} handles a step that applies but cannot be executed as recorded: it replaces the parameters the step carries from the old task with those of the current task (Figure~\ref{fig:intro}).
\emph{Route~C (discard)} leaves the step to the base agent when no recorded step applies.
Route~A addresses failure~(i), Route~B failure~(ii), and Route~C failure~(iii).
To choose a route, \ours{} stores trajectories as paths of a transition graph and retrieves them by task identity.
It separates each step's reusable operation (which page to open, which button to tap) from the parameters it carries from the old task:
names, dates, and amounts, as well as answers and when to finish.
It matches the current step structurally (can it be executed as recorded?) and, only when the current screen has no matching position in the record,
semantically (which recorded step comes next?), with an option to reject all candidates.

\begin{figure}[t]
    \centering
    \includegraphics[width=\linewidth,trim=0 14 0 0,clip]{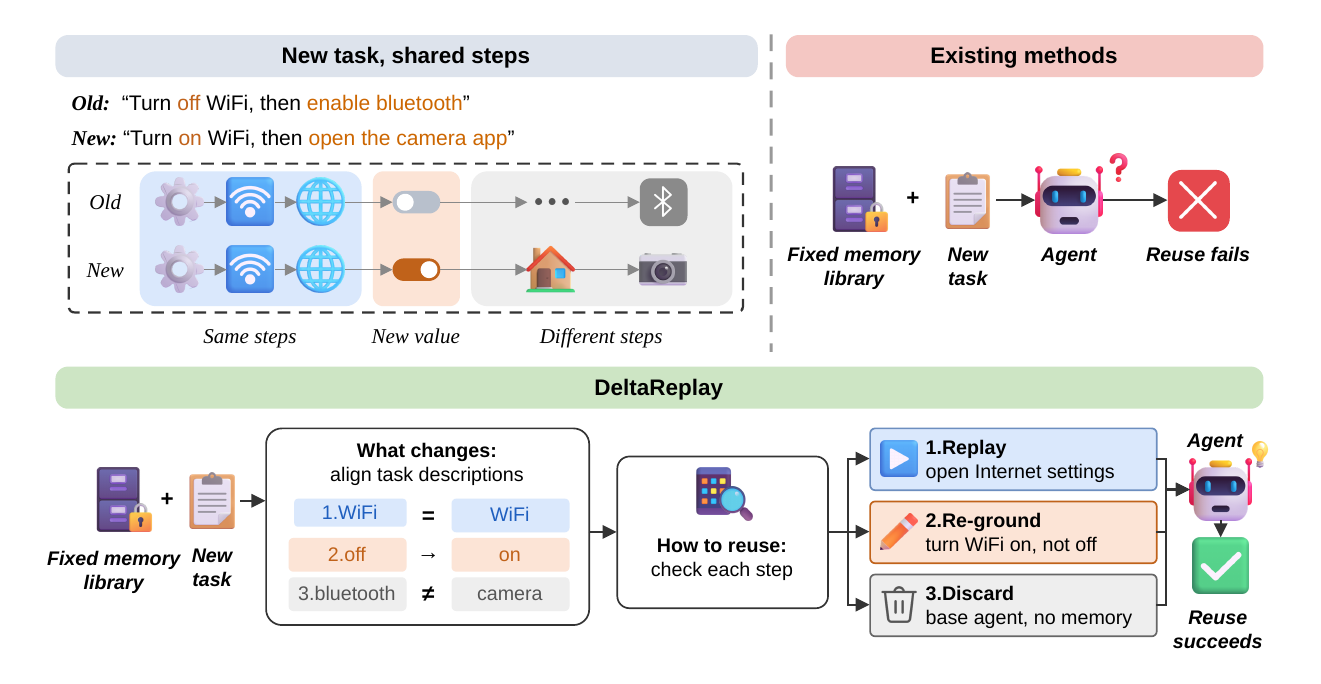}
    \caption{Reusing partly matching memory step by step.
    (Top left) The old and new tasks share their first steps but differ in one value and in the later steps.
    (Top right) Existing methods use or discard a stored record as a whole and fail here.
    (Bottom) \ours{} aligns the two task descriptions and decides at each step whether to replay the step,
    re-ground it with the new value, or discard it.}
    \label{fig:intro}
\end{figure}

Our contributions are as follows:
\begin{itemize}[leftmargin=1.5em,itemsep=1pt,topsep=2pt]
    \item We identify three failure modes of memory reuse in GUI agents (slow reuse of replayable steps, over-discarding of partly usable memory,
    and misuse of inapplicable memory) and trace them mostly to one cause: existing methods treat a memory as a whole,
    whereas reusability varies across the steps of the same memory.
    \item We propose \ours{}, which organizes execution trajectories as a transition graph, splits the action on each edge at reuse time
    into a task-independent operation and task-specific parameters, and chooses at each step among replay, re-grounding,
    and discard according to whether the screen matches and whether the parameters must be replaced.
    \item On 176 tasks from AndroidWorld~\citep{rawles2025androidworld} and SPA-Bench~\citep{chen2024spa},
    \ours{} improves the task success rate over a base agent with the same backbone by 7.8--10.3 and 16.6--25.0 percentage points, respectively,
    and outperforms two memory baselines on AndroidWorld; it cuts per-step latency by up to 57\% when the task is already in memory
    and shows no measurable loss when memory does not apply. Ablations show that verification and discarding are necessary, and route statistics and a step-level annotation characterize how each route is used.
\end{itemize}

\section{Related Work}
\label{sec:related}

\textbf{Memory for GUI agents.}
Prior work stores exploration and app memories~\citep{wen2024autodroid,lee2024mobilegpt,li2024appagent}, workflows~\citep{wang2025agentworkflow}, page graphs and page memories~\citep{chen2025pgagent,kong2025mapagent}, transition-aware knowledge~\citep{xie2025guiexplorer}, state machines and knowledge graphs~\citep{zhong2026actionengine,guan2025kgrag,yu2026graphpilot,chai2026uikobe,qin2026eam}, memory trees~\citep{tan2026enhancing}, demonstrations~\citep{liu2026learnact}, element-level functional memory~\citep{sun2026magnet}, structured, latent, or self-evolving experience~\citep{xiao2026uimem,zhu2026hybrid,li2026echotrailgui,jin2026sega,zhang2026memw}, and tips and shortcuts~\citep{wang2025mobileagente}; MemGUI-Bench evaluates such memories~\citep{liu2026memguibench}.
The closest to ours is MAGNET, which separates what stays stable in memory from what changes by linking the varying appearances of a UI element to its stable function and abstracting similar trajectories into workflows with placeholders~\citep{sun2026magnet};
this separation, however, is made when the memory is built, and the retrieved elements and workflows only serve as references in the prompt.
Overall, these works focus on what to store and how to organize and retrieve it; a retrieved memory is then usually used as a whole, placed in the context or executed in full, and which of its parts fit the new task is rarely judged on its own.

\textbf{How retrieved memory enters the decision.}
Retrieval alone does not settle how its results should be used: in retrieval-augmented generation, CRAG refines, discards, or supplements retrieved documents per query~\citep{yan2024corrective}.
In GUI agents, one family uses memory as context: the model reasons over retrieved records, workflows, or knowledge in its prompt at every step~\citep{xie2025guiexplorer,chen2025pgagent,guan2025kgrag,sun2026magnet}, with refinements that drop untrusted memory blocks through a learned trust gate~\citep{zhang2026focusmem}, let the agent decide when to retrieve~\citep{zhao2026rethinking}, or turn retrieved cases into support for the current candidate actions~\citep{zhang2026beyond}.
Another family executes memory in units. Mobile-Agent-E reuses shortcuts with preconditions and arguments~\citep{wang2025mobileagente}, AppAgentX lets an LLM decide whether a high-level action applies and fill in its arguments~\citep{jiang2025appagentx}, and Darwinian Memory replays whole subtask units indexed by precondition and goal~\citep{mi2026darwinian}. MobiMem records fine-grained interaction sequences to reduce reliance on model inference~\citep{liu2025mobimem}, AgentRR selects experience at different levels of abstraction at run time~\citep{feng2025agentrr}, and EAM extracts an executable path from a knowledge graph and lets an LLM replace its parameters in a single pass before execution~\citep{qin2026eam}. Concurrent EchoPath turns validated trajectories into callable procedures that rebind only inputs declared in advance and are invoked only when they can be replayed deterministically~\citep{zhao2026echopath}. UI-KOBE uses an app graph as guidance: when the current screen is matched to one of its nodes, it chooses among the transitions of that node~\citep{chai2026uikobe}.
Both families, however, decide at the level of a whole memory or a predefined unit: the first leaves the choice of which part to use to the model without any check, and the second fixes which parts may vary when the memory is built or the unit is invoked, whereas reusability varies from step to step and with the new task.

\textbf{Recorded interactions under change.}
Records become outdated as apps change: app graphs must be repaired or re-explored after major UI updates~\citep{chai2026uikobe}, and GUI testing repairs recorded test scripts for new app versions~\citep{pan2020meter,xu2021guider} or migrates tests across apps by semantic matching of GUI events~\citep{behrang2019atm,lin2019craftdroid}, which remains imperfect~\citep{mariani2021semantic}.
These approaches repair a record before reusing it, at the cost of extra device interaction and model calls; \ours{} instead keeps the record fixed and decides at each step which of its steps still hold.

\section{Problem Formulation}
\label{sec:formulation}

We formalize the decision a GUI agent makes at each step of an episode.
At step $t$ the agent observes $s_t=(x_t,h_t,\mathcal{U}_t)$, where $x_t$ is the screenshot, $h_t$ the parsed UI tree,
and $\mathcal{U}_t$ the actionable elements; a candidate generator turns $\mathcal{U}_t$ into the executable action set $\mathcal{A}_t$.
With the instruction $y$ and action history $H_t$, the decision context is
$q_t=(y,H_t,s_t,\mathcal{A}_t)$.
A base policy $\pi_{\mathrm{base}}$ maps $q_t$ to an action in $\mathcal{A}_t$ without memory.
Memory is a read-only graph $\mathcal{G}=(\mathcal{V},\mathcal{E})$ built from recorded successful runs,
in which each run forms a path and each edge $e\in\mathcal{E}$ is one recorded step (Section~\ref{sec:graph}).

\textbf{Executability and applicability.}
Two properties of a retrieved step $e$ determine how it should be used.
Its \emph{executability} $\mu_t(e)\in\mathcal{A}_t\cup\{\bot\}$ is the current action that realizes the recorded action of $e$ as recorded, or $\bot$ if none does or if $e$ carries parameters of the old task that differ from those of $y$, so that executing it as recorded would enter the old values.
Its \emph{applicability} $\rho_t(e)\in\{0,1\}$ indicates whether the operation of $e$ advances the current task $y$ on the current screen.
With $\mathcal{E}_t^{+}=\{e\in\mathcal{E}:\rho_t(e)=1\}$, each step falls into one of three regimes:
\begin{equation}
    c_t=
    \begin{cases}
        \textsc{direct}, & \exists\, e\in\mathcal{E}_t^{+}:\ \mu_t(e)\neq\bot,\\
        \textsc{partial}, & \mathcal{E}_t^{+}\neq\emptyset \ \text{and}\ \mu_t(e)=\bot\ \ \forall e\in\mathcal{E}_t^{+},\\
        \textsc{inapplicable}, & \mathcal{E}_t^{+}=\emptyset.
    \end{cases}
    \label{eq:regimes}
\end{equation}
The correct use of memory differs across regimes: execute $\mu_t(e)$ in the direct regime,
select a new action in $\mathcal{A}_t$ guided by $e$ in the partial regime, and ignore memory in the inapplicable regime.
A step is partial either because it carries parameters of the old task or because its recorded target is missing from the current screen; \ours{} re-grounds the first kind (Route~B) and leaves the second to the base agent.
Neither $\rho_t$ nor $\mu_t$ is observed: the agent knows which task a record was made for, but not whether a given step still applies, so \ours{} estimates both properties at every step.
An agent that supports only some of the three uses fails in at least one regime (Table~\ref{tab:operators}).

\begin{table}[t]
    \centering
    \caption{Behavior of agents with different uses of memory in each regime of Eq.~(\ref{eq:regimes}).}
    \label{tab:operators}
    \small
    \setlength{\tabcolsep}{4pt}
    \resizebox{\linewidth}{!}{%
    \begin{tabular}{lllll}
    \toprule
    Uses supported & \textsc{direct} & \textsc{partial} & \textsc{inapplicable} & Representative work \\
    \midrule
    Prompt only          & \xmark{} reasons anyway    & \cmark{} guides reasoning   & \xmark{} misleads         & GUI-explorer~\citep{xie2025guiexplorer}; \\
                         &                            &                             &                           & MAGNET~\citep{sun2026magnet} \\
    Execute or discard   & \cmark{} executes, fast    & \pmark{} predefined arguments only & \cmark{} falls back       & AppAgentX~\citep{jiang2025appagentx}; \\
                         &                            &                             &                           & Darwinian Memory~\citep{mi2026darwinian} \\
    Re-ground or discard & \xmark{} re-grounds anyway & \pmark{} on known screens   & \cmark{} falls back       & UI-KOBE~\citep{chai2026uikobe} \\
    \midrule
    \ours{} (ours)       & \cmark{} Route A: replay   & \cmark{} Route B: re-ground (parameters) & \cmark{} Route C: discard & This work \\
    \bottomrule
    \end{tabular}}
    \vspace{-5mm}
\end{table}

\section{\ours{}}
\label{sec:method}

Figure~\ref{fig:method} gives an overview, and Algorithm~\ref{alg:routing} (Appendix~\ref{app:algorithm}) lists the per-step procedure.

\begin{figure}[t]
    \centering
    \includegraphics[width=\linewidth,trim=0 20 0 0,clip]{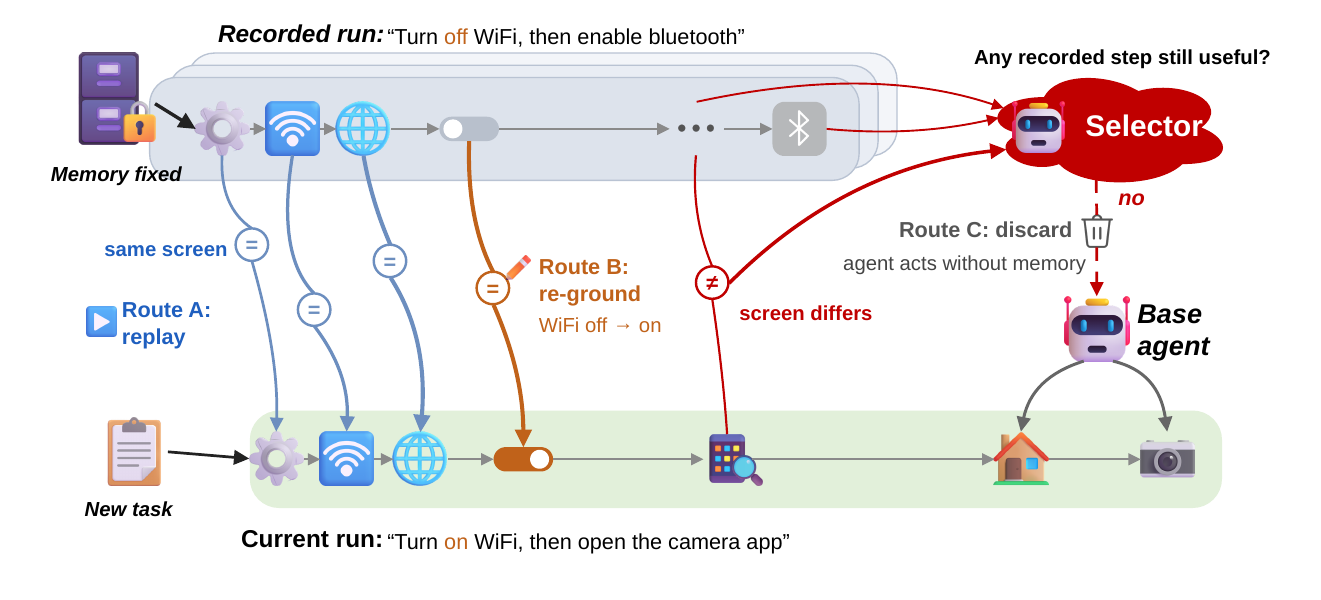}
    \caption{Overview of \ours{}. Memory holds recorded runs (top), retrieved by task identity; the current run (bottom) is a new task.
    Structural matching compares the current screen with the recorded screen at the tracked position ($=$ or $\neq$).
    Route~A (blue) replays a matched step on the corresponding current element; Route~B (orange) re-grounds a matched step that carries a parameter of the old task (WiFi off $\rightarrow$ on).
    When the screen differs, the selector checks whether any recorded step still applies; if none does, Route~C lets the base agent act.}
    \label{fig:method}
\end{figure}

\subsection{Memory as a Transition Graph}
\label{sec:graph}

Each recorded run forms a path in the graph. Each recorded state becomes a node $v=(x_v,h_v)$ holding the screenshot $x_v$
and the UI tree in HTML-like form $h_v$, and each recorded step from $v_i$ to $v_j$ becomes an edge
\begin{equation}
    e=(v_i,v_j,a_e)
\end{equation}
holding the recorded primitive action $a_e$ (type, element index, and any typed text).
Each run also keeps the task description $y_e$ it was recorded for and its task identity (task, template or task family, and app), taken from the benchmark's task identifiers.
At the start of an episode, \ours{} retrieves paths once by task identity:
it takes the records of the same task or the same template when memory holds them;
otherwise it takes the records of easier tasks of the same family (SPA-Bench only) and, if there are none, those of other tasks of the same app (on AndroidWorld, of at most two tasks whose descriptions are closest in wording to the current one).
Each edge has two parts, which determine how far its action generalizes: the \emph{operation} it performs, such as opening a page or tapping a button, which can be reused whenever the screen matches,
and the \emph{parameters} it carries from the old task, such as a typed name, a selected date, an answer, or the decision to finish, which must be replaced for the current task; rules derive this split at run time (Section~\ref{sec:routing}).

\subsection{Structural Matching: Can the Step Be Executed as Recorded?}
\label{sec:structural}

Structural matching (the $=$ and $\neq$ marks in Figure~\ref{fig:method}) checks page-level consistency along the retrieved path: it compares the current screen with the node at the tracked position
by the cosine and structural similarities of their UI trees, each against its own threshold, and also accepts a sufficiently high fuzzy text match on its own;
the thresholds are loose while the agent follows the path and strict when it searches the record for a new matching position.
For the outgoing edge of a matched node, a mapping step (the blue arrows in Figure~\ref{fig:method}) looks for the current element that plays the role of the recorded target:
\begin{equation}
    \hat a_A=\mathrm{VerifyMap}(a_e,h_{v_i},h_t,\mathcal{A}_t)\in\mathcal{A}_t\cup\{\bot\}.
\end{equation}
$\mathrm{VerifyMap}$ matches the recorded target by its structural context rather than by its index, so a target whose index changed
(for example from 23 to 26) is still found.
A non-abstaining result is the estimate of $\mu_t(e)$.
An edge that carries parameters of the old task is not executed as recorded even when its target is found; it goes to Route~B (Section~\ref{sec:routing}).
If the current screen has no matching position in the record, nothing is executed and semantic matching runs; if the screen matches but the mapping fails, the step is left to the base agent (Route~C).

\subsection{Semantic Matching: Does the Step Apply?}
\label{sec:semantic}

Semantic matching (the selector in Figure~\ref{fig:method}) runs only when the current screen has no matching position in the record.
It takes up to $K$ candidate steps $\mathcal{E}_t^K$ whose recorded screens match the current screen under the strict thresholds,
and an LLM selector chooses the correct next step, with an explicit option to reject all candidates:
\begin{equation}
    \hat{e}_t=\mathrm{Select}_{\mathrm{edge}}(q_t,\mathcal{E}_t^K)\in\mathcal{E}_t^K\cup\{\bot\}.
\end{equation}
For records of the same task or template, a selected step restores the position on the path, and later steps are taken without asking the model as long as the screens match;
records borrowed from other tasks of the same app are not followed, and their steps are searched again at every step.
$\hat{e}_t=\bot$ is the estimate $\rho_t=0$ for all candidates.
Thresholds and $K$ are given in Appendix~\ref{app:implementation}.

\subsection{Three Routes}
\label{sec:routing}

The two estimates select how memory enters the current decision.

\textbf{Route A: replay} (verified direct reuse; blue in Figure~\ref{fig:method}).
If structural matching returns $\hat a_A\neq\bot$ for an edge that carries no parameter of the old task, \ours{} executes this verified current action, not the recorded element index.
While the screens keep matching the tracked position, Route~A replaces the base agent's action generation with verification and calls no model;
only a step whose position must be found by semantic matching, such as the first step, costs one selector call.

\textbf{Route B: re-ground} (parameter replacement; orange in Figure~\ref{fig:method}).
A step applies without being executable as recorded when it carries parameters of the old task.
\ours{} decides with four rules, without a model call, whether a step carries old parameters: it answers or declares the task complete;
it types text that appears verbatim in at most one of the two task descriptions; the text of its tapped element contains a word or number
that appears only in the old task description, including the numbers on date and time pickers; or it taps a row of a list whose content is not on the current screen.
\ours{} then aligns the recorded task description $y_e$ with the current instruction $y$ word by word, pairs the unmatched fragments as old and new parameters,
substitutes the new parameters into the typed text and the label of the tapped target, and looks for the element with the new label on the current screen:
\begin{equation}
    \hat a_B=\mathrm{Adapt}(e,y_e,y,\mathcal{A}_t)\in\mathcal{A}_t\cup\{\bot\}.
\end{equation}
If the descriptions cannot be aligned or several elements carry the new label, $\mathrm{Adapt}$ abstains and the base agent performs the step.
Route~B also reuses steps borrowed from records of other tasks of the same app. For these the parameter identification is stricter:
any typed text, any completion, and any tap on an element whose text contains a word of the old task description count as parameters,
so only navigation steps are borrowed. These steps are executed as mapped, like Route~A, and count as Route~B only because they transfer to a task they were not recorded for; the selector (Section~\ref{sec:semantic}) chooses among them or rejects them all.

\textbf{Route C: discard} (base-agent fallback; the ``no'' branch in Figure~\ref{fig:method}).
In all other cases, including a failed mapping, rejection by $\mathrm{Select}_{\mathrm{edge}}$, and abstention by $\mathrm{Adapt}$, memory is dropped and the base agent acts on the same candidate set,
$a_C=\pi_{\mathrm{base}}(q_t)$, with the same prompt and candidate set as without memory.
If its action is equivalent to the next recorded step, including a step whose parameters it had to fill in itself, the position on the path still advances, so later steps can be reused again.
The decision rule is
\begin{equation}
    a_t=
    \begin{cases}
        \hat a_A, & \hat a_A\neq\bot,\\
        \hat a_B, & \hat a_A=\bot,\ \hat a_B\neq\bot,\\
        a_C, & \text{otherwise}.
    \end{cases}
    \label{eq:rule}
\end{equation}
All three routes return an action from $\mathcal{A}_t$, so memory is always tied to the current screen.
Because a replayed or re-grounded step can fail silently, \ours{} does not repeat such a step on an unchanged screen.
$\mathrm{Select}_{\mathrm{edge}}$ and $\pi_{\mathrm{base}}$ are prompted calls to the backbone, $\mathrm{Adapt}$ applies rules, and nothing is trained.
Appendix~\ref{app:theory} gives sufficient conditions for correct transfer and shows that, if correct matches do not hurt,
discarding bounds the per-step loss by the false-acceptance rate.

\section{Experiments}
\label{sec:experiments}

We ask whether \ours{} improves success and latency over the base agent and the memory baselines (RQ1, Section~\ref{sec:main}),
how its gains depend on memory coverage (RQ2, Section~\ref{sec:coverage}), whether it sends steps to suitable routes and those steps are correct (RQ3, Section~\ref{sec:routes}),
and whether verifying and discarding memory are necessary (RQ4, Section~\ref{sec:ablation}).

\subsection{Setup}
\label{sec:setup}

\textbf{Benchmarks and memory.}
AndroidWorld-116 contains one instance of each of the 116 task templates of AndroidWorld~\citep{rawles2025androidworld}.
Its memory holds 42 successful runs of the Qwen3.8-Flash base agent on \emph{other} instances, one for each of 42 templates, recorded before evaluation;
the other 74 templates have no record, and both backbones use this memory (Appendix~\ref{app:benchmarks}), so \ours{} with Qwen3-VL-30B replays runs of the stronger Qwen3.8-Flash.
We group the test tasks by the closest record in memory: \emph{same task} ($n{=}20$), a record of the same template with the same goal text;
\emph{same template} ($n{=}22$), a record of the template with other parameters;
\emph{same app} ($n{=}48$), records of other tasks of the same app only; and \emph{none} ($n{=}26$), no record of the app.
SPA-Bench-60 contains 60 tasks drawn at random from the 150 Chinese single-app tasks of SPA-Bench~\citep{chen2024spa},
after excluding tasks with real-world side effects (payment, social, meeting, and hospital apps) and tasks our phone cannot run (Appendix~\ref{app:benchmarks}); all run on a physical phone.
SPA-Bench tasks come in families of three, each extending the previous one.
Its memory is the union of the successful base-agent runs of both backbones on our SPA-Bench task pool, 50 runs of 42 tasks (27 of them in SPA-Bench-60), so each backbone can also replay runs of the other; the groups are
\emph{same task} ($n{=}27$), a record of the task itself;
\emph{same family} ($n{=}15$), a record of an easier task of the same family, which shares the first steps; and
\emph{no family record} ($n{=}18$), no record of the task or of an easier task of its family (12 of these tasks' apps hold records of other families).
The agent is not given the group label, but it retrieves records by the benchmark's task identifiers (Section~\ref{sec:graph}), so the kind of record an episode receives follows its group, and the selector prompt names that kind (Appendix~\ref{app:prompts}).

Because memory contains only successful base-agent runs, coverage groups are correlated with task difficulty.
On AndroidWorld, every template in the same-task and same-template groups was solved by the Qwen3.8-Flash base agent on another instance.
On SPA-Bench-60, memory includes the base-agent runs reported in Tables~\ref{tab:main} and~\ref{tab:groups},
and 32 of the 33 tasks without a record of their own were failed by both backbones in those runs.
Success rates are therefore not comparable across groups, and we report each group separately.

\textbf{Backbones and baselines.}
We use Qwen3-VL-30B-A3B-Instruct~\citep{bai2025qwen3}, a vision-language mixture-of-experts (MoE) model with about 30B total and 3B activated parameters, which reads the screenshot and the UI tree,
and Qwen3.8-Flash~\citep{qwen2026flashnext}, the production version of Qwen3.8-Flash-Next, an MoE model with 125B parameters (6B activated per token).
The \emph{base agent} uses the same backbone, prompts, and candidate-action interface as \ours{}, without memory.
\emph{GUI-Owl} is GUI-Owl-1.5~\citep{xu2026mobileagentv35}, a native GUI agent model without memory.
\emph{RAG-prompting} adds the records most similar to the current task and screen to the base agent's prompt;
\emph{GUI-explorer}~\citep{xie2025guiexplorer} mines transition-aware knowledge, i.e., what an action does, and our adaptation injects the per-step entries most similar to the current task and screen into the prompt.
Both use Qwen3.8-Flash and the same memory as \ours{} (Appendix~\ref{app:baselines}).

\textbf{Metrics.}
We report task success rate (SR) and latency (Lat.), the mean wall-clock time per executed step, excluding the time spent on the device (Appendix~\ref{app:implementation}).
Success is decided by an LLM judge for all methods: the single-app evaluator of SPA-Bench~\citep{chen2024spa} with Qwen3.8-Flash as the judge model, which reads the screenshots of the episode.
We use the same judge on both benchmarks because SPA-Bench has no rule-based checkers; on the AndroidWorld tasks whose checkers work in our harness,
those checkers agree with our success labels on 91.5\% of the episodes (Appendix~\ref{app:benchmarks}).

\textbf{Run-to-run variation.}
Each task is run once per method, and decoding through the cloud APIs is not fully deterministic.
On the 26 AndroidWorld tasks whose app has no record in memory (\emph{none}), \ours{} retrieves nothing and acts through the base agent at every step,
yet its success differs from the base agent's by one task with each backbone and its latency by 2.1\,s (Qwen3-VL-30B) and 4.1\,s (Qwen3.8-Flash) per step (Table~\ref{tab:groups}).
We therefore regard differences of about one task per group, or about 2--4\,s per step, as within run-to-run variation.

\subsection{Main Results (RQ1)}
\label{sec:main}

\begin{table}[t]
    \centering
    \caption{Overall results (SR: task success rate; Lat.: latency per step). Best in bold, second-best underlined;
    Qwen3-VL-30B is Qwen3-VL-30B-A3B-Instruct; --: not run.}
    \label{tab:main}
    \small
    \setlength{\tabcolsep}{5pt}
    \begin{tabular}{llcccc}
    \toprule
    \multirow{2}{*}{Method} & \multirow{2}{*}{Backbone} & \multicolumn{2}{c}{AndroidWorld-116} & \multicolumn{2}{c}{SPA-Bench-60} \\
    \cmidrule(lr){3-4}\cmidrule(lr){5-6}
     & & SR (\%) $\uparrow$ & Lat. (s) $\downarrow$ & SR (\%) $\uparrow$ & Lat. (s) $\downarrow$ \\
    \midrule
    Base agent    & Qwen3-VL-30B   & 43.1 & \textbf{9.3}  & 23.3 & \textbf{10.3} \\
    Base agent    & Qwen3.8-Flash  & 49.1 & 24.8 & 41.7 & 21.0 \\
    GUI-Owl       & GUI-Owl-1.5    & 50.9 & \underline{9.6} & -- & -- \\
    RAG-prompting & Qwen3.8-Flash  & 41.4 & 34.8 & -- & -- \\
    GUI-explorer  & Qwen3.8-Flash  & 47.4 & 19.0 & -- & -- \\
    \midrule
    \ours{}       & Qwen3-VL-30B   & \underline{53.4} & 11.3 & \underline{48.3} & \underline{12.2} \\
    \ours{}       & Qwen3.8-Flash  & \textbf{56.9} & 19.7 & \textbf{58.3} & 20.9 \\
    \bottomrule
    \end{tabular}
    \vspace{-5mm}
\end{table}

Table~\ref{tab:main} shows that \ours{} achieves the highest success rate on both benchmarks.
With the same backbone, it improves over the base agent by 10.3 (Qwen3-VL-30B) and 7.8 points (Qwen3.8-Flash) on AndroidWorld
and by 25.0 and 16.6 points on SPA-Bench.
With Qwen3.8-Flash it is also faster than its base agent on AndroidWorld (19.7\,s vs.\ 24.8\,s per step), partly because replayed steps skip
model reasoning; with the already fast Qwen3-VL-30B, matching adds about 2\,s per step.
Both memory baselines use the same records but fall below the base agent in success (41.4\% and 47.4\% vs.\ 49.1\%),
and RAG-prompting is the slowest agent (34.8\,s per step).
\ours{} also surpasses GUI-Owl-1.5 (50.9\%), a model trained for GUI control, with either general-purpose backbone.

\subsection{Results by Memory Coverage (RQ2)}
\label{sec:coverage}

\begin{table}[t]
    \centering
    \caption{Results by the closest record in memory; best per column in bold, second-best underlined.
    On SPA-Bench-60, memory includes the base-agent runs shown here, so its same-family and no-family-record groups
    consist mostly of tasks both base agents failed (Section~\ref{sec:setup}).}
    \label{tab:groups}
    \small
    \setlength{\tabcolsep}{3.1pt}
    \begin{tabular}{llcccccccc}
    \toprule
    \multicolumn{2}{l}{\emph{AndroidWorld-116}}  & \multicolumn{2}{c}{Same task (20)} & \multicolumn{2}{c}{Same template (22)} & \multicolumn{2}{c}{Same app (48)} & \multicolumn{2}{c}{None (26)} \\
    \cmidrule(lr){3-4}\cmidrule(lr){5-6}\cmidrule(lr){7-8}\cmidrule(lr){9-10}
    Method & Backbone & SR & Lat. & SR & Lat. & SR & Lat. & SR & Lat. \\
    \midrule
    Base agent    & Qwen3-VL-30B  & 70.0 & 8.8 & 63.6 & \textbf{7.8} & 31.3 & \underline{10.1} & 26.9 & \textbf{9.2} \\
    Base agent    & Qwen3.8-Flash & 80.0 & 21.7 & 63.6 & 19.5 & \underline{39.6} & 27.9 & \underline{30.8} & 23.1 \\
    GUI-Owl       & GUI-Owl-1.5   & 75.0 & \underline{7.1} & \underline{81.8} & 10.8 & 37.5 & \textbf{9.9} & \underline{30.8} & \underline{9.5} \\
    RAG-prompting & Qwen3.8-Flash & 80.0 & 32.2 & 77.3 & 31.8 & 25.0 & 38.0 & 11.5 & 32.8 \\
    GUI-explorer  & Qwen3.8-Flash & 80.0 & 19.0 & \textbf{86.4} & 15.9 & 29.2 & 20.7 & 23.1 & 18.2 \\
    \midrule
    \ours{}       & Qwen3-VL-30B  & \underline{85.0} & \textbf{6.2} & \underline{81.8} & \underline{10.4} & \underline{39.6} & 12.8 & \underline{30.8} & 11.3 \\
    \ours{}       & Qwen3.8-Flash & \textbf{90.0} & 9.3 & \textbf{86.4} & 12.6 & \textbf{41.7} & 24.8 & \textbf{34.6} & 19.0 \\
    \midrule
    \multicolumn{2}{l}{\emph{SPA-Bench-60}} & \multicolumn{2}{c}{Same task (27)} & \multicolumn{2}{c}{Same family (15)} & \multicolumn{2}{c}{No family record (18)} & & \\
    \cmidrule(lr){3-4}\cmidrule(lr){5-6}\cmidrule(lr){7-8}
    Method & Backbone & SR & Lat. & SR & Lat. & SR & Lat. & & \\
    \midrule
    Base agent    & Qwen3-VL-30B  & 51.9 & \textbf{9.2} & 0.0 & \underline{11.6} & 0.0 & \textbf{10.7} & & \\
    Base agent    & Qwen3.8-Flash & \underline{88.9} & 17.6 & 0.0 & 19.2 & 5.6 & 25.2 & & \\
    \ours{}       & Qwen3-VL-30B  & 81.5 & \underline{9.7} & \underline{33.3} & \textbf{10.3} & \underline{11.1} & \underline{17.6} & & \\
    \ours{}       & Qwen3.8-Flash & \textbf{92.6} & 15.3 & \textbf{46.7} & 21.8 & \textbf{16.7} & 24.8 & & \\
    \bottomrule
    \end{tabular}
    \vspace{-5mm}
\end{table}

\textbf{Memory covers the task.}
On AndroidWorld (Table~\ref{tab:groups}), \ours{} gains the most when memory holds the same task or template.
With Qwen3.8-Flash, it raises success from 80.0\% to 90.0\% (same task) and from 63.6\% to 86.4\% (same template),
while cutting latency from 21.7\,s to 9.3\,s and from 19.5\,s to 12.6\,s.
With Qwen3-VL-30B, same-task steps also get faster (8.8\,s to 6.2\,s), but same-template steps get slower (7.8\,s to 10.4\,s).
With Qwen3.8-Flash on these 42 tasks, only \ours{} solves 8 and only the base agent solves 1
(exact McNemar test, $p=0.039$; Appendix~\ref{app:paired}).
Prompting-based baselines also profit from such memory (GUI-explorer ties \ours{} at 86.4\% on same template)
but reason through every step, at 15.9--32.2\,s per step (failure~i).

\textbf{Memory covers only part of the task.}
In the SPA-Bench same-family group, the record comes from an easier task that shares only the first steps.
Neither base agent solves any of these 15 tasks, whereas \ours{} solves 33.3\% and 46.7\%,
consistent with reusing the shared steps.

\textbf{Memory covers nothing applicable.}
When memory holds only other tasks of the same app, or nothing from the app, the right decision is mostly to discard.
\ours{} then stays at or above the base agent with both backbones (e.g., 41.7\% vs.\ 39.6\% and 34.6\% vs.\ 30.8\% with Qwen3.8-Flash),
whereas RAG-prompting drops to 25.0\% and 11.5\% and GUI-explorer to 29.2\% and 23.1\%.
Retrieved records that do not fit the task can mislead agents that always condition on them (failure~iii).
On SPA-Bench, where 12 of the 18 no-family-record tasks have records of other families of the same app, \ours{} still gains 11.1 points with each backbone.

\subsection{Route Usage and Step Correctness (RQ3)}
\label{sec:routes}

Table~\ref{tab:routeshare} (Appendix~\ref{app:routestats}) shows that \ours{} uses memory in proportion to coverage.
Replay (Route~A) takes 56.1\% of the steps when memory holds the same task and 41.5\% when it holds the same template, and never fires otherwise.
Route~B takes 14.8\% of the steps in the same-template group, where records differ in their parameters,
and 10.1\% in the same-app group, where it borrows steps from other tasks.
With no record of the app, every step is discarded.

To check whether routed steps are correct, an LLM judge (Claude Fable 5.1), blind to the route, labeled 150 executed steps, 50 per route (Appendix~\ref{app:annotation}).
Route~A and Route~B steps reach a precision of 83.7\% and 78.0\% (Table~\ref{tab:annotation}), compared with 75.5\% for Route~C steps;
with about 50 steps per route, these differences are within sampling error (95\% intervals of about $\pm$10--12 points).
Route~A and Route~B steps also contain fewer harmful actions (2 and 3 vs.\ 6), although Route~C also handles the harder steps where no memory applies.
Most of their errors are no-ops, such as tapping a target clipped off-screen, which waste a step but leave the task recoverable.
Within Route~B, parameter-adapted steps are almost always correct (19 of 20), whereas borrowed steps are harder (20 of 30).

\subsection{Ablations (RQ4)}
\label{sec:ablation}

\begin{table}[t]
    \centering
    \caption{Ablations (Qwen3.8-Flash, AndroidWorld) on the tasks each ablation affects; variants in Appendix~\ref{app:ablation}.
    Calls: model calls per step; A/B/C: share of executed steps per route.}
    \label{tab:ablation}
    \small
    \setlength{\tabcolsep}{5pt}
    \begin{tabular}{lccc}
    \toprule
    Variant & SR (\%) & Calls & A/B/C (\%) \\
    \midrule
    \multicolumn{4}{l}{\emph{Same template} ($n{=}22$)} \\
    Base agent                          & 63.6 & 2.11 & 0/0/100 \\
    \ours{} (full)                      & \textbf{86.4} & 1.00 & 41/15/44 \\
    \quad Naive replay                  & 9.1  & 0.00 & 100/0/0 \\
    \quad w/o Route~C (no discard)      & 45.5 & 0.17 & 68/32/0 \\
    \midrule
    \multicolumn{4}{l}{\emph{Same app} ($n{=}30$ of 48)} \\
    Base agent                          & 43.3 & 2.03 & 0/0/100 \\
    \ours{} (full)                      & \textbf{50.0} & 2.19 & 0/14/86 \\
    \quad w/o semantic matching         & 20.0 & 0.60 & 0/73/27 \\
    \bottomrule
    \end{tabular}
    \vspace{-5mm}
\end{table}

In Table~\ref{tab:ablation}, we remove parts of the routing decision.
\textbf{Unchecked replay fails.}
Executing records verbatim, without the screen check, target mapping, parameter adaptation, or fallback, solves only 9.1\% of same-template tasks:
the records are correct for their own instance but wrong for this one.
\textbf{Verification and discarding are necessary.}
Without Route~C and the screen check, the best memory candidate is always executed and success falls to 45.5\%, below the base agent;
in the same-app group, dropping the screen check and the selector lets Route~B take 73\% of the steps and lowers success to 20.0\%:
applicability must be decided before re-grounding.
\textbf{Replay saves computation.}
The full system uses 1.00 model calls per step on same-template tasks, less than half of the base agent's 2.11.
When memory rarely applies (same app), matching adds 0.16 calls per step.

\section{Conclusion and Limitations}
\label{sec:conclusion}

\ours{} chooses for each retrieved step whether to replay it, re-ground it, or discard it,
instead of supporting only some of these uses.
On AndroidWorld and SPA-Bench, this choice improves success over same-backbone base agents,
can make the agent faster when memory covers the task, and avoids the losses that prompting-based memory suffers when memory does not fit.

\textbf{Limitations.}
First, a recorded step may also fail to apply because its target has moved or disappeared, for example after an app update~\citep{lu2025transbench}.
\ours{} leaves such steps to the base agent, and our evaluation does not vary the app version.
Second, Route~B finds the values to replace by aligning the two task descriptions word by word, so a value that the instruction only implies, such as a date given as ``tomorrow'', is also left to the base agent.
Third, records are retrieved by the benchmark's task identifiers, as this work studies how retrieved records are used.
Retrieving them from the instruction alone is left to future work.
Finally, task success is decided by an LLM judge, and the step-level labels come from another LLM (Claude Fable 5.1); neither has been checked by humans.

\newpage
\subsection*{AI Use Statement}
Large language models are part of the object of study: Qwen3-VL-30B-A3B-Instruct, Qwen3.8-Flash, and GUI-Owl-1.5 serve as agent backbones and baselines.
Qwen3.8-Flash also serves as the task-success judge on both benchmarks (Section~\ref{sec:setup}; Appendix~\ref{app:benchmarks}).
We used Claude Fable 5.1 as a blind annotator of 150 executed agent steps (Section~\ref{sec:routes}; protocol in Appendix~\ref{app:annotation}).
In this work, we used generative AI tools to help develop the conceptual framework, propose and refine hypotheses, design and give feedback on the method and experiments,
formulate the propositions in Appendix~\ref{app:theory} and write their proofs, implement the method, the baselines, and the experiment scripts,
clean and reformat the benchmark task lists and run records, interpret results, and translate text.
Generating synthetic data sets and qualitative or thematic data analysis are not applicable to this work.
Additionally, we used generative AI tools to search for, identify, and summarize related literature, suggest the structure and title of the paper, draft and edit the text, create figures,
and check bibliographic metadata against arXiv and CrossRef.
We have reviewed all AI-assisted work: the authors checked every claim, proof, and cited work, verified and tested the code, and checked every reported number against the experiment logs.
We take responsibility for the final content of this work, including text, claims, and artifacts produced with the aid of generative AI.

\subsection*{Ethics Statement}
This work studies memory reuse for mobile GUI agents.
Reliable reuse can reduce manual effort for routine phone tasks and help users who find small touch interfaces hard to operate,
but more capable GUI automation also carries risks: unintended actions in sensitive apps, handling of private data, and unauthorized automation.
Our experiments use benchmark tasks, run SPA-Bench on the authors' own app accounts, and involve no human subjects;
memory stores no data of other users, and the recorded screens, which may show the authors' account data, are not released.
Deployed systems should confirm sensitive actions with the user, log agent actions, and avoid storing private information in memory without consent.

\subsection*{Reproducibility Statement}
Section~\ref{sec:method} and Algorithm~\ref{alg:routing} (Appendix~\ref{app:algorithm}) specify the method;
Appendix~\ref{app:implementation} lists hyperparameters and implementation details, and Appendix~\ref{app:prompts} gives the prompts;
Appendix~\ref{app:benchmarks} describes the task sets, coverage groups, memory construction, and the task-success judge;
Appendices~\ref{app:baselines}--\ref{app:annotation} describe the baselines, ablation variants, route statistics, and annotation protocol.
The ancillary files of this arXiv submission contain our code, including the full prompts, and the task lists of both benchmarks.

\bibliography{references}
\bibliographystyle{iclr2027_conference}

\newpage
\appendix

\section{Routing Algorithm}
\label{app:algorithm}

Algorithm~\ref{alg:routing} gives the per-step procedure of \ours{}.

\begin{algorithm}[h]
\caption{\ours{}: per-step decision}
\label{alg:routing}
\begin{algorithmic}[1]
\Require instruction $y$, history $H_t$, current state $s_t$, candidate actions $\mathcal{A}_t$, records retrieved by task identity for this episode, base policy $\pi_{\mathrm{base}}$
\Ensure action $a_t\in\mathcal{A}_t$
\State \textit{// Structural matching: does the current screen match the node at the tracked position?}
\State $e=(v_i,v_j,a_e)\gets$ outgoing edge at the tracked position in the records
\If{there is no tracked position, the records are borrowed from other tasks, or $s_t$ does not match $v_i$ under the loose thresholds}
    \State \textit{// Semantic matching: no matching position in the record}
    \State $e\gets\mathrm{Select}_{\mathrm{edge}}(q_t,\mathcal{E}_t^K)$, where $\mathcal{E}_t^K$ are candidate steps whose screens match $s_t$ under the strict thresholds
\EndIf
\If{$e\neq\bot$}
    \If{$e$ carries no parameter of the old task}
        \State $\hat a\gets\mathrm{VerifyMap}(a_e,h_{v_i},h_t,\mathcal{A}_t)$ \Comment{Route A (Route B for a borrowed step)}
    \ElsIf{$e$ comes from the same task or template}
        \State $\hat a\gets\mathrm{Adapt}(e,y_e,y,\mathcal{A}_t)$ \Comment{Route B: re-ground}
    \Else
        \State $\hat a\gets\bot$ \Comment{borrowed steps that carry parameters are not reused}
    \EndIf
    \If{$\hat a\neq\bot$ and $\hat a$ was not just executed on an unchanged screen}
        \State \Return $\hat a$
    \EndIf
\EndIf
\State $a_C\gets\pi_{\mathrm{base}}(q_t)$ \Comment{Route C: discard}
\State advance the tracked position if $a_C$ matches the next recorded step
\State \Return $a_C$
\end{algorithmic}
\end{algorithm}

\section{Implementation Details}
\label{app:implementation}

\textbf{Memory graph.}
Before evaluation, recorded runs are stored as paths of a read-only graph and then frozen.
Each node stores the HTML-like UI description and the screenshot path of a recorded screen, and each edge stores the recorded action;
each run keeps its task description and task identity (task, template or task family, and app), taken from the benchmark's task identifiers.

\textbf{Per-step pipeline.}
At each step the agent dumps the UI tree, converts it into an HTML-like description, builds the candidate action set, and records an annotated screenshot.
Screens are compared with three similarities: cosine similarity (of bge-small-zh-v1.5 embeddings of the visible text of the HTML-like descriptions),
structural overlap (Jaccard overlap of the two screens' widget-structure sets), and fuzzy-matching similarity (rapidfuzz fuzzy text ratio, 0--100).
Two screens match when both the cosine and the structural score reach their thresholds, or when the fuzzy-matching score alone is at least 80.
While following a record we use loose thresholds (cosine $\geq0.60$, structural $\geq0.70$);
when searching the record for a new matching position we use strict thresholds (cosine $\geq0.65$, structural $\geq0.80$).
When the current screen has no matching position in the record, at most $K=8$ candidate steps are kept (each action at a given recorded step is kept only once);
the selector receives the goal, the task history, the current HTML, and each candidate step as its recorded action with the text on the element it operates on
(Appendix~\ref{app:prompts}), and returns one candidate or rejects all.
An episode ends after $\max(8,\lceil 3\,n_{\mathrm{ref}}\rceil)$ steps, where $n_{\mathrm{ref}}$ is the reference step count of the task; on AndroidWorld it is also capped at 15 minutes.
All backbones decode at temperature 0, and Qwen3.8-Flash runs in thinking mode.
Decoding through the cloud APIs is not fully deterministic, so repeated runs of the same policy can differ; each task is run once per method (Section~\ref{sec:setup}).
If the selector rejects all candidates, the mapping fails, or $\mathrm{Adapt}$ abstains,
the base agent acts on the same candidate set.

\textbf{Compute.}
SPA-Bench-60 runs on a physical Android phone, a Redmi Note 8 (Android 11, MIUI 12.5).
Three base-agent episodes (Qwen3-VL-30B: calendar\_0, gaodeditu\_0, gaodeditu\_7) ran on a second Redmi Note 8 running Android 9; all other SPA-Bench episodes ran on the Android 11 device.
AndroidWorld runs on the emulator recommended by AndroidWorld, a Pixel 6 AVD with Android 13 (API 33, google\_apis x86\_64 system image),
here with 6 CPU cores, 4\,GB of RAM, and a $1080\times2400$ display.
All backbones are called through the cloud APIs of Alibaba Cloud Bailian, directly or through a relay service, and no local GPU is used;
the bge-small-zh-v1.5 screen encoder runs locally on CPU.
Reported latency excludes the time spent on the device (UI-tree capture, screenshots, action execution, and post-action waiting); it includes model inference or API time and action parsing.

\section{Prompts}
\label{app:prompts}

The selector and the base agent each send the backbone a single user message, with no system message; the full prompt texts are in the code provided as ancillary files.

\textbf{Base agent.} The prompt of $\pi_{\mathrm{base}}$, which Route~C also uses, gives the user goal, the history of actions, the HTML-like description of the current screen, and the candidate actions, each with an index and a JSON description.
It asks the model to choose exactly one candidate that best advances the goal, with guidelines to rely on visible evidence, avoid repeating an action that made no progress, close overlays that block the screen, and declare completion only when the goal is visibly satisfied,
and to reply with a JSON object that names the chosen candidate and the element it acts on.

\textbf{Selector.} Semantic matching sends the prompt below; placeholders and header labels are in italics, and long lines are wrapped at $\hookrightarrow$.
Each candidate line gives the recorded step's position in its run, its action type, the text of its target element and any typed text or other arguments, and the first texts of its recorded screen.
The history lists each earlier step of the episode with its action and a short description of its effect, which the backbone writes after each step from the screens before and after it;
for a step taken from memory, the entry is the recorded step as it appears in the candidate list.
The prompt opens with a header that names the kind of record the candidates come from, set at retrieval by comparing the task identifiers of the record and of the current task, for example:
\lstset{basicstyle=\ttfamily\scriptsize,breaklines=true,breakatwhitespace=true,columns=fullflexible,frame=single,literate={'}{{\textquotesingle}}1,
  escapeinside={(*}{*)},postbreak=\mbox{\ensuremath{\hookrightarrow}\space}}
\begin{lstlisting}
(*\textrm{\textit{Same task:}}*)
You resume a recorded SUCCESSFUL run of exactly the current goal on a phone.
(*\textrm{\textit{Same template, other parameters:}}*)
You resume a recorded SUCCESSFUL run of the same task with other parameters on a phone.
The run was recorded for: (*\textrm{\textit{recorded goal}}*)
Names, texts, numbers and dates in the recorded steps belong to that run, not to the current goal;
never pick a step before the current goal's own values that it depends on have been entered.
(*\textrm{\textit{Easier task of the same family:}}*)
You resume a recorded SUCCESSFUL run of an easier task of the same task family on a phone.
The easier task was recorded for: (*\textrm{\textit{recorded goal}}*)
Its goal is the first part of the current goal: after its last recorded step the current goal still
needs more steps, so never pick a candidate that ends the task.
(*\textrm{\textit{Another task of the same app:}}*)
You resume a recorded SUCCESSFUL run of ANOTHER task in the same app on a phone.
The run was recorded for: (*\textrm{\textit{recorded goal}}*)
Only its way through the app (opening the app, reaching a screen, opening a form or a list) can be the
current goal's next step; its typed values, its choices and its end belong to that other task:
never pick a candidate that enters that task's values or ends the task, and output -1 when the
current goal needs something different here.
\end{lstlisting}
The full prompt is:
\begin{lstlisting}
(*\textrm{\textit{header}}*)
Goal:
(*\textrm{\textit{goal}}*)

Actions already taken in this episode:
(*\textrm{\textit{history}}*)

Current screen (HTML):
(*\textrm{\textit{HTML-like description of the current screen}}*)

Candidates are recorded steps of that run whose recorded screen looks like the current screen.
Pick the candidate that is the correct NEXT step from the current screen, given the actions already taken
(for example, do not repeat a step whose effect is already visible, do not skip a step whose effect is missing).
Output -1 if none of them is the correct next step.

Candidates (index | recorded step):
(*\textrm{\textit{one line per candidate, at most $K=8$}}*)

Return EXACTLY one JSON object on one line, no markdown: {"reason": "<short>", "candidate_index": <number>}
\end{lstlisting}

\section{Benchmarks and Memory Construction}
\label{app:benchmarks}

\textbf{AndroidWorld-116.}
One instance of each of the 116 AndroidWorld task templates forms the test set.
Before evaluation, the Qwen3.8-Flash base agent ran on other instances of the templates, generated with a different random seed.
Memory keeps every successful run, except runs that repeat the same action three or more times in a row (six for scrolls) or wait in more than 40\% of their steps:
42 of 49 successful runs, one each for 42 templates. It is frozen before evaluation and shared by both backbones and both memory baselines.
Coverage groups are defined by the closest record memory holds for the task:
\emph{same task} (a record with the same goal text, from another instance of the template),
\emph{same template} (a record of the template with other parameters),
\emph{same app} (no record of the template, but records of other tasks of the same app), and
\emph{none} (no record of the app).
The same-task and same-template groups therefore contain only templates the base agent solved at least once.

\textbf{SPA-Bench-60.}
The 60 tasks are drawn from the 150 Chinese single-app tasks of SPA-Bench.
We exclude tasks with real-world side effects: Alipay (payments), WeChat (posts), Tencent Meeting (real meetings), a hospital app, and one QQ Music task that shares to a WeChat contact.
We hold out 17 tasks as an unused reserve and exclude tasks our phone cannot run: Huawei Browser is not available, China Telecom tasks need an account we do not have, and JD shows a CAPTCHA at every launch.
We also exclude tasks that do not complete a run within a fixed retry budget and three tasks without a complete base-agent run.
From the remaining 84 tasks we draw 60 at random with a fixed seed, stratified over three strata of our task pool (19/30/11 tasks), which are unrelated to the coverage groups;
two apps that show a login or verification wall on our phone (Haodaifu, Meituan) are dropped with all their tasks, which are replaced in a random order fixed in advance.
SPA-Bench tasks come in families of three that extend each other, so a run of an easier task shares the first steps of a harder one.
Memory is the union of the successful base-agent runs of both backbones on our SPA-Bench task pool, which also contains tasks outside SPA-Bench-60.
The \emph{same task} group is the set of tasks at least one backbone solved when memory was built
(Qwen3.8-Flash solved 24 of these 27, Qwen3-VL-30B 14).
The \emph{same family} group holds tasks whose own run is not in memory but an easier task of the same family is.
The \emph{no family record} group holds tasks for which neither the task nor an easier task of its family is in memory; 12 of its 18 tasks belong to apps with records of other families.
Of the 33 tasks in the last two groups, 32 were failed by both backbones when memory was built,
so base-agent SR there is low by construction.
Tables~\ref{tab:awtasks} and~\ref{tab:spatasks} list the tasks of each group.

\textbf{Task success.}
Both benchmarks and all methods use the same judge, the single-app evaluator of SPA-Bench~\citep{chen2024spa}, with PaddleOCR for screen text and Qwen3.8-Flash as the judge model.
The judge receives the screenshot of every executed step with the executed action drawn on it, and a final screenshot taken 3\,s after the last action;
when an episode has more than 10 screenshots, the earlier ones are summarized in batches of 10, and the judge receives the summaries and the last 10 screenshots.
It returns success or failure.
We do not apply SPA-Bench's key-component check, which requires given UI strings to appear on screen, because these strings were collected on older app versions.
Each episode is judged once, whether or not the agent declared completion; an episode stopped by a time limit without a verdict counts as a failure.
AndroidWorld's rule-based checkers are recorded but not used for scoring.
In our setup they cannot score 41 of the 116 tasks.
The checkers of 25 question-answering tasks read the agent's answer through AndroidWorld's own environment loop, in which our agents do not run;
the checkers of the other 16 read the screen through that loop, require the task's app to still be in the foreground although our harness closes the apps before they run, or check a file rename that our task descriptions do not ask for.
On the other 75 tasks, the checkers agree with our success labels on 91.5\% (248/271) of the episodes of GUI-Owl, RAG-prompting, GUI-explorer, and \ours{} (Qwen3.8-Flash).

\begin{table}[p]
    \centering
    \setlength{\tabcolsep}{4pt}
    \renewcommand{\arraystretch}{0.9}
    \caption{AndroidWorld-116 tasks by coverage group.}
    \label{tab:awtasks}
    \footnotesize
    \begin{tabular}{@{}lp{0.8\linewidth}@{}}
    \toprule
    Group & Tasks\tabularnewline
    \midrule
    Same task (20) & \raggedright BrowserMaze, CameraTakePhoto, CameraTakeVideo, ClockStopWatchPausedVerify, ClockStopWatchRunning, ExpenseDeleteDuplicates, ExpenseDeleteDuplicates2, SystemBluetoothTurnOff, SystemBluetoothTurnOffVerify, SystemBluetoothTurnOn, SystemBluetoothTurnOnVerify, SystemBrightnessMaxVerify, SystemBrightnessMinVerify, SystemWifiTurnOff, SystemWifiTurnOffVerify, SystemWifiTurnOnVerify, TasksCompletedTasksForDate, TasksHighPriorityTasks, TasksIncompleteTasksOnDate, TurnOffWifiAndTurnOnBluetooth\tabularnewline
    Same template (22) & \raggedright ContactsAddContact, ContactsNewContactDraft, ExpenseAddSingle, ExpenseDeleteMultiple, ExpenseDeleteSingle, MarkorCreateNote, MarkorCreateNoteFromClipboard, NotesIsTodo, NotesRecipeIngredientCount, NotesTodoItemCount, OsmAndFavorite, OsmAndMarker, SimpleCalendarAddOneEvent, SimpleCalendarAddOneEventInTwoWeeks, SimpleCalendarAddRepeatingEvent, SimpleCalendarDeleteOneEvent, SimpleCalendarLocationOfEvent, SimpleSmsReply, SimpleSmsReplyMostRecent, SimpleSmsSend, SportsTrackerActivitiesCountForWeek, TasksHighPriorityTasksDueOnDate\tabularnewline
    Same app (48) & \raggedright BrowserDraw, BrowserMultiply, ClockTimerEntry, ExpenseAddMultiple, ExpenseAddMultipleFromGallery, ExpenseAddMultipleFromMarkor, ExpenseDeleteMultiple2, MarkorAddNoteHeader, MarkorChangeNoteContent, MarkorCreateFolder, MarkorCreateNoteAndSms, MarkorDeleteAllNotes, MarkorDeleteNewestNote, MarkorDeleteNote, MarkorEditNote, MarkorMergeNotes, MarkorMoveNote, MarkorTranscribeReceipt, MarkorTranscribeVideo, NotesMeetingAttendeeCount, OsmAndTrack, SimpleCalendarAddOneEventRelativeDay, SimpleCalendarAddOneEventTomorrow, SimpleCalendarAnyEventsOnDate, SimpleCalendarDeleteEvents, SimpleCalendarDeleteEventsOnRelativeDay, SimpleCalendarEventOnDateAtTime, SimpleCalendarEventsInNextWeek, SimpleCalendarEventsInTimeRange, SimpleCalendarEventsOnDate, SimpleCalendarFirstEventAfterStartTime, SimpleCalendarNextEvent, SimpleCalendarNextMeetingWithPerson, SimpleSmsResend, SimpleSmsSendClipboardContent, SimpleSmsSendReceivedAddress, SportsTrackerActivitiesOnDate, SportsTrackerActivityDuration, SportsTrackerLongestDistanceActivity, SportsTrackerTotalDistanceForCategoryOverInterval, SportsTrackerTotalDurationForCategoryThisWeek, SystemBrightnessMax, SystemBrightnessMin, SystemCopyToClipboard, SystemWifiTurnOn, TasksDueNextWeek, TasksDueOnDate, TurnOnWifiAndOpenApp\tabularnewline
    None (26) & \raggedright AudioRecorderRecordAudio, AudioRecorderRecordAudioWithFileName, FilesDeleteFile, FilesMoveFile, OpenAppTaskEval, RecipeAddMultipleRecipes, RecipeAddMultipleRecipesFromImage, RecipeAddMultipleRecipesFromMarkor, RecipeAddMultipleRecipesFromMarkor2, RecipeAddSingleRecipe, RecipeDeleteDuplicateRecipes, RecipeDeleteDuplicateRecipes2, RecipeDeleteDuplicateRecipes3, RecipeDeleteMultipleRecipes, RecipeDeleteMultipleRecipesWithConstraint, RecipeDeleteMultipleRecipesWithNoise, RecipeDeleteSingleRecipe, RecipeDeleteSingleWithRecipeWithNoise, RetroCreatePlaylist, RetroPlayingQueue, RetroPlaylistDuration, RetroSavePlaylist, SaveCopyOfReceiptTaskEval, SimpleDrawProCreateDrawing, VlcCreatePlaylist, VlcCreateTwoPlaylists\tabularnewline
    \bottomrule
    \end{tabular}

    \caption{SPA-Bench-60 tasks by coverage group.}
    \label{tab:spatasks}
    \footnotesize
    \begin{tabular}{@{}lp{0.78\linewidth}@{}}
    \toprule
    Group & Tasks\tabularnewline
    \midrule
    Same task (27) & \raggedright QQMusic\_1, QQMusic\_3, bilibili\_0, bilibili\_3, bilibili\_6, bilibili\_7, cainiaoguoguo\_0, cainiaoguoguo\_1, calendar\_0, deskclock\_0, deskclock\_1, gaodeditu\_0, gaodeditu\_3, hanglvzongheng\_3, hanglvzongheng\_4, hanglvzongheng\_5, qvnaer\_0, qvnaer\_1, qvnaer\_3, taobao\_1, taobao\_3, tencentdocs\_0, xiaohongshu\_0, xiaohongshu\_1, xiaohongshu\_3, youdaodict\_0, zhihu\_5\tabularnewline
    Same family (15) & \raggedright QQMusic\_4, QQMusic\_5, bilibili\_1, bilibili\_4, bilibili\_5, cainiaoguoguo\_2, calendar\_1, calendar\_2, deskclock\_2, gaodeditu\_2, gaodeditu\_4, taobao\_4, tencentdocs\_1, toutiao\_4, zhihu\_1\tabularnewline
    No family record (18) & \raggedright douyin\_3, douyin\_5, eleme\_0, eleme\_3, eleme\_4, eleme\_5, gaodeditu\_7, gaodeditu\_8, hanglvzongheng\_1, settings\_0, settings\_1, taobao\_0, taobao\_6, taobao\_7, taobao\_8, toutiao\_0, toutiao\_1, zhihu\_3\tabularnewline
    \bottomrule
    \end{tabular}
\end{table}
\FloatBarrier

\section{Baselines}
\label{app:baselines}

\textbf{Base agent.} The same backbone, prompts, and candidate-action interface as \ours{}, without memory and without the repeat guard.

\textbf{GUI-Owl-1.5.} A native GUI agent model from Mobile-Agent-v3.5~\citep{xu2026mobileagentv35}, which receives the screenshot.
We call GUI-Owl-1.5 through Alibaba Cloud Bailian under the service name gui-plus-2026-02-26; the service does not disclose its parameter count.
To compare decision-making alone, it uses the same candidate-action interface as the base agent rather than the official Mobile-Agent-v3.5 prompt and action parser:
at each step it reads the screenshot and the UI tree text and selects one of the candidate actions by its index.
Decoding is at temperature 0.

\textbf{RAG-prompting.} The Qwen3.8-Flash base agent with the retrieved records added to its prompt.
RAG-prompting uses the same memory as \ours{}, including the records of the task's own template.
At each step the model first summarizes the current page into a task intent, a page summary, and an app intent.
The top 5 memory pages are retrieved by $0.35\times$task-intent $+\,0.50\times$page-summary $+\,0.15\times$app-intent similarity;
the page summaries and intents of memory pages were written by an LLM when memory was built.
The model then infers the current subtask, and the top 5 recorded steps leaving those pages are retrieved by the similarity between the subtask and each step's description.
There is no threshold or memory gate, so records are added to the prompt at every step.
Each record lists the step description written at recording time, the action description (including the recorded element index), the match similarity, the goal text of the source task, and its run, page, and step identifiers.
The model picks one of the current candidate actions; if it picks none, or no subtask can be inferred, the base agent acts on the same candidate set.

\textbf{GUI-explorer.} We adapt \citet{xie2025guiexplorer}, which mines per-element knowledge by autonomous exploration, retrieves it by visual element patches,
and ranks it with an MLLM: here transition-aware knowledge is mined from the recorded runs, retrieved for the current task and screen from the text UI tree,
and injected into the action prompt of the Qwen3.8-Flash agent through the same candidate-action interface.
The base agent and all baselines choose their actions from the same list of candidate actions generated from the UI tree.
The knowledge comes from the same frozen memory as \ours{}: on AndroidWorld, 42 successful runs with 316 steps, each step becoming one entry.
An entry holds the screen text before and after the step, the task description, the action type, and the step description that the recording agent wrote for that step, which serves as the function description.
We replace the screenshot-based element retrieval of GUI-explorer with text retrieval over the UI tree:
entries are ranked by $0.7\times$page-text similarity $+\,0.3\times$task similarity, and the action types and function descriptions of the top 5 are added to the prompt as a separate section.
The adaptation also differs from the original method in two ways: the knowledge is mined from existing records rather than a separate exploration phase,
and function descriptions reuse the recorded step descriptions rather than being distilled by a separate MLLM.

\section{Ablation Variants}
\label{app:ablation}

All variants use Qwen3.8-Flash, the protocol of Table~\ref{tab:main}, and the fixed task-set denominator.
\textbf{Full:} Route~A replays the record after a screen check, Route~B re-grounds a recorded step by replacing the parameters it carries from the old task,
and Route~C lets the base agent act when memory offers no applicable step.
\textbf{Naive replay:} the record is executed verbatim (recorded element indices and text), with no screen check, no grounding, and no parameter adaptation.
\textbf{w/o Route~C:} no base-agent fallback and no screen check; the best memory candidate is always executed, and the episode ends with a
status-complete step when memory offers nothing (this terminating step calls no model and is not counted in A/B/C).
\textbf{w/o semantic matching:} the screen check and the selector are both removed and only the target mapping is kept; the first borrowed step whose target element grounds on the current screen is executed without checking the screen or consulting the selector.
The same-template ablations use all 22 tasks; the same-app ablation uses the first 30 of the 48 same-app tasks in a fixed order.

\section{Route Statistics}
\label{app:routestats}

Table~\ref{tab:routeshare} gives the share of executed steps taken by each route in each coverage group of AndroidWorld-116 (Section~\ref{sec:routes}).
Steps borrowed from other tasks of the same app count as Route~B: although they are executed as mapped, they are transferred to a task they were not recorded for, and only their operation is reused.
Under Eq.~(\ref{eq:regimes}) they are in the direct regime; counted as Route~A, they would move the 10.1\% of same-app steps from Route~B to Route~A.

\begin{table}[h]
    \centering
    \caption{Share of executed steps taken by each route (\ours{}, Qwen3.8-Flash, AndroidWorld-116), by coverage group.}
    \label{tab:routeshare}
    \small
    \begin{tabular}{lccc}
    \toprule
    Group & A (\%) & B (\%) & C (\%) \\
    \midrule
    Same task     & 56.1 & 0.0  & 43.9 \\
    Same template & 41.5 & 14.8 & 43.7 \\
    Same app      & 0.0  & 10.1 & 89.9 \\
    None          & 0.0  & 0.0  & 100.0 \\
    \midrule
    All           & 13.7 & 8.1  & 78.2 \\
    \bottomrule
    \end{tabular}
\end{table}

\section{Step-Level Annotation Protocol}
\label{app:annotation}

We sampled 150 executed steps of \ours{} (Qwen3.8-Flash, AndroidWorld), 50 per route, with at most four steps per task for Routes~A and~B and two for Route~C.
The annotator, Claude Fable 5.1, was shown the goal, the screen before the action with the target marked, the screen after the action, and the action,
and was blind to the route.
It labeled each action \emph{correct} (a sensible step toward the goal from that screen, or a correct completion or answer), \emph{incorrect}, or \emph{unclear}.
Incorrect steps are split into \emph{no-ops} (the screen did not change, e.g., because the target was clipped off-screen) and
\emph{harmful} actions (wrong element, content, or direction).
Precision is the number of correct steps divided by the number of steps that could be judged.
Route~B mixes parameter-adapted replays (20 steps, 19 correct) and steps borrowed from another task's record (30 steps, 20 correct).
Table~\ref{tab:annotation} reports the labels per route.

\begin{table}[h]
    \centering
    \caption{Blind step-level annotation of 150 executed steps of \ours{} (Qwen3.8-Flash, AndroidWorld; 50 per route); precision is correct steps over judged steps.}
    \label{tab:annotation}
    \small
    \begin{tabular}{lcccccc}
    \toprule
    Route & Steps & Corr. & No-op & Harm. & Uncl. & Prec. (\%) \\
    \midrule
    A: replay     & 50 & 41 & 6 & 2 & 1 & 83.7 \\
    B: re-ground  & 50 & 39 & 8 & 3 & 0 & 78.0 \\
    C: discard    & 50 & 37 & 6 & 6 & 1 & 75.5 \\
    \midrule
    All           & 150 & 117 & 20 & 11 & 2 & 79.1 \\
    \bottomrule
    \end{tabular}
\end{table}

\section{Paired Comparison}
\label{app:paired}

On the 42 AndroidWorld tasks whose template is in memory (same task and same template), we compare \ours{} with the Qwen3.8-Flash base agent task by task.
The base agent alone solves $b=1$ task and \ours{} alone solves $c=8$, a difference of $+16.7$ points.
Under the null hypothesis of equal success, the $b+c=9$ discordant tasks split as $\mathrm{Binomial}(9,\tfrac12)$,
and the exact two-sided McNemar test gives $p=2\sum_{k=0}^{1}\binom{9}{k}2^{-9}=20/512\approx0.039$.

\section{When Is Route-Based Reuse Safe?}
\label{app:theory}

We give sufficient conditions under which a re-grounded step transfers correctly, and show that the discard route bounds negative transfer.
The analysis is local to one recorded step and does not assume that the recorded and current UI graphs are isomorphic.

\textbf{Setup.}
Let $F(s,a)$ be the next state after executing action $a$ in state $s$, and let $\phi(s)$ be the task-relevant semantic state.
A recorded step $e=(v_i,v_j,a_e)$ executes $a_e$ in state $s_i$ and reaches $s_j=F(s_i,a_e)$.
Consider a shift $\tau=(\tau_S,\tau_A)$ that maps the states and action of this step to the current episode.
We call $\tau$ \emph{transition-consistent on $e$} if
\begin{equation}
    \phi(\tau_S(s_i))=\phi(s_i),\qquad \phi(\tau_S(s_j))=\phi(s_j),\qquad F(\tau_S(s_i),\tau_A(a_e))=\tau_S(F(s_i,a_e)).
    \label{eq:consistency}
\end{equation}
The first two conditions require semantic invariance; the third requires the local dynamics to commute with the shift.

\begin{proposition}[Transfer from a local correspondence]
\label{prop:transfer}
Let $\tau$ be transition-consistent on $e$. If the current state is $\bar s=\tau_S(s_i)$ and re-grounding returns $\hat a=\tau_A(a_e)\in\mathcal{A}_t$,
then executing $\hat a$ reaches the recorded semantic successor: $\phi(F(\bar s,\hat a))=\phi(s_j)$.
\end{proposition}
\begin{proof}
$\phi(F(\bar s,\hat a))=\phi(F(\tau_S(s_i),\tau_A(a_e)))=\phi(\tau_S(F(s_i,a_e)))=\phi(\tau_S(s_j))=\phi(s_j)$,
using transition consistency and then semantic invariance at the successor.
\end{proof}

Route~A is the special case $\tau_A(a_e)=\mu_t(e)$, in which the current action realizes the recorded one as recorded;
Route~B covers shifts in which $\tau_A$ changes the parameters the action carries, which is why the recorded action cannot be copied but can be re-grounded with the parameters of the current task.

\begin{proposition}[Discard bounds negative transfer]
\label{prop:safety}
Let per-step utility lie in $[0,1]$ and compare the routed policy with $\pi_{\mathrm{base}}$ on the same contexts.
Let $\mathcal{C}$ be the event that an accepted match (Route~A or~B) is correct, $\mathcal{W}$ the event that an incorrect match is accepted,
and $\mathcal{R}$ the discard event, with $p=\Pr(\mathcal{C})$, $\delta=\Pr(\mathcal{W})$, and $g=\mathbb{E}[U_{\mathrm{route}}-U_{\mathrm{base}}\mid\mathcal{C}]$.
Then $J_{\mathrm{route}}-J_{\mathrm{base}}\geq p\,g-\delta$.
\end{proposition}
\begin{proof}
On $\mathcal{R}$, Route~C executes $\pi_{\mathrm{base}}$, so the conditional difference is zero.
On $\mathcal{W}$, the difference is at least $-1$. Summing over the three events gives the bound.
\end{proof}

When correct matches do not hurt ($g\geq0$), negative transfer is bounded by the false-acceptance rate $\delta$.
The ablations in Section~\ref{sec:ablation} are consistent with this view: removing Route~C or semantic matching raises $\delta$
by forcing matches that should be rejected, and success falls below the base agent.
These are idealized sufficient conditions, not guarantees that the prompted matching and selection modules satisfy them;
they identify two testable requirements: executable grounding of a local correspondence and a low false-acceptance rate.

\end{document}

%% file: math_commands.tex
\usepackage{amsmath,amsfonts,bm}

\def\eqref#1{equation~\ref{#1}}

\def\1{\bm{1}}

\DeclareMathAlphabet{\mathsfit}{\encodingdefault}{\sfdefault}{m}{sl}
\SetMathAlphabet{\mathsfit}{bold}{\encodingdefault}{\sfdefault}{bx}{n}

